\documentclass{article}

\usepackage{arxiv}

\usepackage[utf8]{inputenc} 
\usepackage[T1]{fontenc}    
\usepackage{hyperref}       
\usepackage{url}            
\usepackage{booktabs}       
\usepackage{amsfonts}       
\usepackage{microtype}      

\title{Decision Variance in Risk-Averse Online Learning}

\author{
Sattar Vakili, ~Alexis Boukouvalas\\
Prowler.io\\
Cambridge, UK \\
{\tt\small \{sattar,alexis\}@prowler.io}\\
\And
Qing Zhao\\
School of Electrical and Computer Engineering\\ Cornell University\\ 
Ithaca, NY, USA\\
{\tt\small qz16@cornell.edu}\\
}

\newtheorem{theorem}{Theorem}
\newtheorem{lemma}{Lemma}

\newtheorem{proof}{Proof}

\def\nn{\nonumber}

\def\MVL{\mbox{\footnotesize MV-LCB}}
\def\CBE{\mbox{\footnotesize CB-AE}}
\def\MVFL{\mbox{\footnotesize MV-FL}}

\def\mubar{\bar{\mu}}
\def\F{\mathcal{F}}
\def\Or{\mathcal{O}}
\def\X{\mathcal{X}}

\def\E{\mathbb{E}}
\def\R{\mathbb{R}}
\def\N{\mathbb{N}}
\def\I{\mathbb{I}}
\def\nn{\nonumber}
\def\l{\lceil}
\def\r{\rceil}
\def\argmin{\texttt{argmin}}

\def\KL{\mathrm{KL}}

\def\sigmabar{\bar{\sigma}^2}
\def\K{\mathcal{K}}

\def\Gahat{\widehat{\Gamma}}

\def\MV{\mbox{\footnotesize MV}}
\def\MVbar{\bar{\texttt{MV}}}
\def\Reg{R_{\pi}}
\def\Sumnotkstar{\sum_{k\in[K]\setminus k^*}}

\usepackage{amsmath}
\usepackage{stackengine}
\usepackage{amssymb}
\usepackage{algorithm}
\usepackage{algpseudocode}
\usepackage{color}
\usepackage{graphicx}
\usepackage{subfigure}

\begin{document}
\maketitle

\begin{abstract}
Online learning has traditionally focused on the expected rewards. In this paper, a risk-averse online learning problem under the performance measure of the mean-variance of the rewards is studied. Both the bandit and full information settings are considered. The performance of several existing policies is analyzed, and new fundamental limitations on risk-averse learning is established. In particular, it is shown that although a logarithmic distribution-dependent regret in time $T$ is achievable (similar to the risk-neutral problem), the worst-case (i.e. minimax) regret is lower bounded by $\Omega(T)$ (in contrast to the $\Omega(\sqrt{T})$ lower bound in the risk-neutral problem). This sharp difference from the risk-neutral counterpart is caused by the the variance in the player's decisions, which, while absent in the regret under the expected reward criterion, contributes to excess mean-variance due to the non-linearity of this risk measure. The role of the decision variance in regret performance reflects a risk-averse player's desire for robust decisions and outcomes.  
\end{abstract}


\keywords{Online Learning \and Multi Armed Bandit \and Full Information \and Risk Averse}

\section{Introduction}

\subsection{Risk-Neutral Online Learning}

Consider an online decision making problem with a finite set $[K]=\{1,2,\dots,K\}$ of actions and a learner who chooses the actions sequentially. Each chosen action $k\in[K]$ at time $t$  results in a random reward $X_{k,t}$ drawn independently over time from an unknown distribution. 

Classic formulations of the problem target at the \emph{expected} cumulative reward over a horizon of length $T$. A commonly adopted performance measure is regret defined as the cumulative reward loss in expectation as compared to the optimal policy with the knowledge of the reward distribution under each action. A sublinear regret order in $T$ implies that not knowing the reward distributions results in diminishing reward loss per play, and the specific regret order gives a finer measure on the efficiency of the learning policies. 

We are yet to specify the observations available to the learner for decision-making at each time. Two common feedback models have been considered in the literature: the full-information setting and the bandit setting (see, for example, \cite{Price}). In the former, after taking an action $X_{k,t}$ at time $t$, the random rewards of all $K$ actions are revealed to the learner. This feedback model applies to applications such as stock investment and portfolio management. In the latter, only the reward of the chosen action $k$ is revealed. This model arises naturally from applications such as online ads placement where the payoff of a particular action is only observed after the action is tried out. This coupling between information gathering and reward earning under the bandit setting leads to the exploration-exploitation tradeoff that significantly complicates the problem. 

When comparing learning policies in their regret performance, there are two approaches to handling the bias toward specific reward distributions (consider, for example, a policy that always chooses action $1$; it works perfectly when this action does lead to the highest expected reward). In the first approach, only policies offering uniformly good performance across all reward distributions (in a certain class) are admissible. These admissible policies are then compared under each possible set of reward distributions. Such a distribution-dependent regret typically depends on certain statistics of the underlying reward distributions
such as the Kullback-Leibler (KL) divergence  and the gap in the mean values. In the second approach, all policies are admissible. The performance of a policy, however, is taken as the worst among all reward distributions. The regret (referred to as the worst-case or minimax regret) of a policy is thus independent of specific distributions, and policies are compared at different reward distributions, i.e., their specific worst scenarios. It is known that in the full-information setting, the distribution-dependent regret and the minimax regret are lower bounded by $\Omega(\log K)$ \cite{Mourtada} and $\Omega(\sqrt T)$ \cite{PLG}, respectively, with order-optimal policies given in \cite{War14,PLG}. In the bandit setting, the distribution-dependent regret and the minimax regret are lower bounded by $\Omega(K\log T)$ \cite{Lai&Robbins85AAM} and $\Omega(\sqrt{KT})$ \cite{BR,NonSto}, respectively, with order-optimal policies given in, for example, \cite{NonSto,Auer&etal02ML,DSEE}.

\subsection{Risk-Averse Online Learning and Main Results}

In this paper, we consider risk-averse online learning. We adopt Markowitz's mean-variance measure, a common risk measure especially for modern portfolio selection \cite{Finance}. The mean-variance of a random variable $X$ is defined as
\begin{equation}\label{MVD}
\MV(X)=\sigma^2(X)-\lambda\mu(X),
\end{equation}
a linear combination of its mean $\mu(X)$ and variance $\sigma^2(X)$~\cite{Marko}. The parameter $\lambda$ is the risk-tolerance factor. It can be interpreted as the inverse Lagrangian multiplier in the constrained optimization of maximizing the expected return $\mu(X)$ subject to a given variance level.

Let $\{\pi_t\}_{t=1}^T$ denote the sequence of actions chosen by a policy $\pi$ and $X_{\pi_t,t}$ the reward obtained at time $t$ under action $\pi_t$. The objective is to minimize the cumulative risk given by the total mean-variance: 
\begin{eqnarray}\nn
\MV_\pi(T) = \sum_{t=1}^T \MV(X_{\pi_t,t}). 
\end{eqnarray}
The above cumulative mean-variance measure is an extension of the risk measure of a random variable $X$ to a risk measure of a random process $\{X_{\pi_t,t}\}_{t=1}^T$. In particular, the risk constraint on the variance is imposed locally for each time $t$. This is particularly relevant to applications such as clinical trial where the risk in each action (i.e. for each patient) needs to be controlled. 

Similar to the risk-neutral online learning, regret is defined as the excess in cumulative mean-variance in comparison to the optimal policy $\pi^*$ under known reward distributions:
\begin{eqnarray}\nn
\Reg(T) = \MV_\pi(T) - \MV_{\pi*}(T).
\end{eqnarray}

The regret definition in risk-averse online learning is similar to the one in risk-neutral online learning except that the measure of expected value is replaced with the measure of mean-variance. In the risk-neutral setting, due to the linearity of the expectation operator (and by Wald first identity), regret can be expressed as a weighted sum of the expected number of times suboptimal actions are chosen where the weights are the suboptimality gap of the corresponding action. 
In the risk-averse setting, however, due to the non-linearity of the performance measure, regret is no longer merely determined by the mean-variance of the rewards of the selected actions, but importantly also, as shown in Sec.~\ref{Sec:LB}, by the variance in the decisions; hence, the title of the paper. 
Under the mean-variance measure, in addition to choosing the suboptimal actions, the uncertainty in the actions with different outcomes is penalized, which is motivated by learner's interest in robust decisions and outcomes.  

In Sec.~\ref{Sec:LB}, we establish fundamental limits on the performance of policies under the mean-variance measure. Specifically, we show that the impact of decision variance on the distribution-dependent regret is absorbed by the leading constants of the regret. In other words, the same $\Omega(K\log T)$ and $\Omega(\log K)$ lower bounds on distribution-dependent regret holds under the mean-variance risk measure for bandit and full information cases, respectively.  In contrast and rather surprisingly, the
variance in the decisions makes an $\Omega(T)$ worst-case regret inevitable under both bandit and full-information feedback models, which is striking in comparison to the sublinear regret order of $\Omega(\sqrt{T})$ in the corresponding risk-neutral problems. 

We also analyze the performance of several policies under the risk-averse measure. 
In the bandit setting, we consider Mean-Variance Lower Confidence Bound (MV-LCB), a modification of the classic UCB introduced in~\cite{Auer&etal02ML} for risk-neutral bandits, and Confidence Bounds based Action Elimination (CB-AE), a more structured policy based on an action elimination method introduced in~\cite{UCBRev} for risk-neutral bandits. CB-AE considerably reduces the regret by reducing the variance in the decisions. 
We show that, while an $\Or(K\log T)$ distribution-dependent regret is achievable, both MV-LCB and CB-AE have a linear worst-case regret in time. In parallel, in the full information case, we study a modification of Follow the Leader policy~\cite{War14}, referred to as MV-FL as well as CB-AE. We show that, while an $\Or(\log K)$ distribution-dependent regret is achievable, both MV-FL and CB-AE have a linear worst-case regret in time. The analysis of the policies shows the tightness of the lower bound results.

\subsection{Related Work}



In contrast to the long history of extensive studies on risk-neutral online learning dating back to Thompson's work in 1933 \cite{Thompson33}, risk-averse online learning is receiving research attention only fairly recently.  
A couple of existing studies have extended the mean-variance measure to the bandit problem. In defining the mean-variance of a random reward sequence under a given policy, two other approaches exist in the literature, which we refer to as the empirical risk constraint and the global risk constraint. Together with the local risk constraint considered in this work, these models target different applications, depending on which type of uncertainty is deemed as risk.  In the empirical risk constraint model first introduced in~\cite{Sani}, temporal fluctuations over the \emph{empirical} mean of the \emph{realized} reward sequence are deemed undesired (e.g. volatility in financial security). The risk measure is given by the empirical mean and empirical variance of the realized reward sequence. The global risk constraint model  concerns with only the variance of the total reward seen at the end of the time horizon (e.g. retirement investment). The risk measure is thus given by the mean-variance of the sum of the rewards.

The first and yet incomplete study of the empirical risk constraint model was given in~\cite{Sani}, which established an $\Or(\sqrt T)$ upper bound on distribution-specific and an $\Or(T^{2/3})$ upper bound on distribution-independent regrets. The upper bound of $\Or(\sqrt{T})$ on the distribution-specific regret offered by MV-UCB is loose, and no result on achievable lower bounds was given in~\cite{Sani}. The result for the empirical risk constraint model was completed in~\cite{Vakili2015} with lower bounds of $\Omega(\log T)$ for distribution-specific regret and $\Omega(T^{2/3})$ for minimax regret, as well as a tight analysis of MV-UCB showing its optimal $\Theta(\log T)$ distribution-specific regret. Incomplete studies of the global risk constraint model have been reported in~\cite{VZAllerton15}. But regret lower bounds remain open, without which, the optimality of policies cannot be assessed. 

This work gives the first and complete set of results on local risk constraint model: problem-specific and minimax, full-information and bandit feedbacks, lower bounds and order-optimal policies. Local risk constraint is fundamentally different from empirical and global risk constraints. The differences in objective functions lead to different regret expressions, different feasible minimax regret orders ($T^{\frac{2}{3}}$ vs. linear), and different techniques used in analysis.  

In~\cite{GRAMAB}, the
quality of an action was measured by a general function of the mean and the variance of the
random variable. Authors in~\cite{Dima} considered an online variance minimization model.
The model in~\cite{Dima} is different than ours in that it allows for linear actions that distribute a budget over actions at each time (i.e. choose a weighted sum of the actions), which differs from the atomic actions in our model. Note such linear actions can reduce variance (e.g. a linear combination of two i.i.d. random variables has a lower variance than both). Also, \cite{Dima} assumed direct observation of the variance instead of the value of random rewards. These studies are closer to the risk-neutral bandit problems than to the problem studied in this
paper in that the variance in the decisions does not effect
the regret as it dominantly does in
our results.


In~\cite{VZAllerton15,Safety}, bandit problem under the measure of
value at risk
was studied. In~\cite{Safety}, learning policies using the
measure of conditional value at risk were developed. However, the performance guarantees were
still within the risk-neutral framework (in terms of the
loss in the expected total reward)
under the assumption that the best action in terms of the mean value is also the best action in terms of
the conditional value at risk. Logarithm of moment generating function was considered as a risk measure for bandit problems in~\cite{Mill} and high probability
bounds on regret were obtained. We point out that the logarithm of the moment generating function reduces to mean-variance for a random variable with Gaussian distribution. Even under this special case,~\cite{Mill} uses the mean-variance conditioned on the action at each $t$, thus measures only randomness in the reward itself for a fixed action, but not the randomness in actions which has complex dependencies on past observations. Thus,~\cite{Mill} is close to the risk-neutral case and has similar regret bounds, while this work shows drastically different bounds.

We point out that both bandit and full information problems have been studied under a different, the so-called adversarial setting where the reward process is non-stochastic and designed adversarially. 
Under a full information setting,~\cite{Even} considered a linear combination of mean and empirical standard deviation (in contrast to mean-variance) and established a negative result showing the infeasibility of sublinear regret. The adversarial setting is fundamentally different than the stochastic setting in the assumptions and solution methods.  

\section{Problem Formulation and Preliminaries}\label{Sec:PF}
Consider a stochastic online learning problem with a discrete set $[K]=\{1,2,\dots,K\}$ of actions. At each time $t$, a learner chooses an action $k\in [K]$ and receives the corresponding reward $X_{k,t}$, drawn from an unknown distribution $f_{k}$. The rewards are independent over $k$, and i.i.d. over $t$. Let $\F=\{f_{k}\}_{k=1}^K$ denote the set of distributions. We use $\E_\F$ and $\Pr_\F$ to denote the expectation and probability with respect to $\F$ and drop the subscript $\F$ when it is clear from the context.  
Let $\mu_k$, $\sigma^2_k$ and $MV_k$ denote the mean, variance and mean-variance of the random reward $X_k$ of action $k$.

An action selection policy $\pi$ specifies a sequence of mappings $\{\pi_t\}_{t\ge 1}$ from the history of observations to the action to choose at each time $t$. In the bandit information setting the learner only observes the reward of the selected action at each time, thus, we have $\pi_t:[K]^{t-1}\times\R^{t-1}\rightarrow [K]$. In the full information setting, the learner observes the rewards of all actions at each time, thus we have $\pi_t:[K]^{t-1}\times\R^{K\times (t-1)}\rightarrow [K]$.  

The objective is an action selection policy $\pi$ that minimizes regret defined with respect to the optimal policy $\pi^*$ under known reward distributions:
\begin{eqnarray}\label{RegretDef}
\Reg(T) = \sum_{t=1}^T\MV(X_{\pi_t,t}) -\sum_{t=1}^T \MV(X_{\pi^*_t,t}),
\end{eqnarray}
where
$\pi_t$ denotes the action taken by policy $\pi$ at time $t$, and $\MV(\cdot)$ denotes the mean-variance of a random variable as defined in \eqref{MVD}. We point out that different from the risk-neutral case where the optimal policy $\pi^*$ under known reward distributions is easily known to be a single-action policy, the corresponding statement cannot be easily made under the mean-variance measure.

\paragraph{Concentration Inequalities}
Most existing work on risk-averse (e.g. \cite{Sani, Vakili2015}) and risk-neutral (\cite{Auer&etal02ML, DSEE}) online learning assume bounded support distribution. We assume the random variable $(X_{k,1}-\mu_k)^2-\sigma_k^2$, for all $k$, is sub-Gaussian with parameter $b^2$, i.e., its moment generating function is bounded by that of a Gaussian distribution with variance $b^2$:
\begin{eqnarray}\nn
\E\left[\exp\bigg(u \left((X_{k,1}-\mu_k)^2-\sigma_k^2\right) \bigg)\right] \le \exp(\frac{u^2b^2}{2}).
\end{eqnarray}
As a result of the Chernoff-Hoeffding bound~(\cite{SubG}), we have the concentration inequalities on the sample mean and the sample mean-variance given in Lemma~\ref{Lemma:CHB}. This class includes all distributions (of action rewards) with bounded support. The extension to light-tailed distributions is fairly standard as similar concentration inequalities exist for light-tailed distributions (e.g. see~\cite{DSEE,Bubeck13}). 

Let $\I[.]$ denote the indicator function that is, for an event $\mathcal{E}$,  $\I[\mathcal{E}]=1$ if and only if $\mathcal{E}$ is true, and $\I[\mathcal{E}]=0$, otherwise. 
Let $\tau_{k,t}=\sum_{s=1}^t\I[\pi_s=k]$ denote the number of times that action $k$ has been chosen until time $t$. 
The sample mean, the sample variance\footnote{The use of the biased estimator for the variance is for the
simplicity of the expression. The results presented in this work remain the same with the use of
the unbiased estimator with $\tau_{k,t}$ ($t$) replaced by $\tau_{k,t}-1$ ($t-1$) in the expression of $\sigmabar_{k,t}$ under bandit (full information) setting.} and the sample mean-variance of each action $k$ up to time $t$ are, respectively, denoted by $\mubar_{k,t}$, $\sigmabar_{k,t}$ and $\MVbar_{k,t}=\sigmabar_{k,t}-\lambda\mubar_{k,t}$.  Specifically, under bandit information 
$\mubar_{k,t}=\frac{1}{\tau_{k,t}}\sum_{s=1}^{t}\I[\pi_s=k]X_{k,s}$ and $\sigmabar_{k,t}=\frac{1}{\tau_{k,t}}\sum_{s=1}^{t}\I[\pi_s=k](X_{k,s}-\mubar_{k,t})^2$; and, under full information $\mubar_{k,t}=\frac{1}{t}\sum_{s=1}^{t}X_{k,s}$ and $\sigmabar_{k,t}=\frac{1}{t}\sum_{s=1}^{t}(X_{k,s}-\mubar_{k,t})^2$. To keep the notation uncluttered we drop the specification of the policy from $\tau_{k,t}$, $\mubar_{k,t}$, $\sigmabar_{k,t}$ and $\MVbar_{k,t}$. 

\begin{lemma}[\cite{Vakili2015}]\label{Lemma:CHB}
Let $\MVbar_t$ be the sample mean-variance of a random variable $X$ obtained from $t$ i.i.d. observations. Let $\mu=\mathbb{E}[X]$, $\sigma^2=\mathbb{E}[(X-\mu)^2]$, and assume that $(X-\mu)^2-
\sigma^2$ has a sub-Gaussian distribution, i.e.,
\begin{eqnarray}\nn
\mathbb{E}[e^{u((X-\mu)^2-\sigma^2)}]\le e^{\zeta_1 u^2/2}
\end{eqnarray}
for some constant $\zeta_1>0$.
As a result $X-\mu$ has a sub-Gaussian distribution, i.e.,
\begin{eqnarray}\nn
\mathbb{E}[e^{u(X-\mu)}]\le e^{\zeta_0 u^2/2}.
\end{eqnarray}
Let $\zeta=\max\{\zeta_0,\zeta_1\}$. We have, for all constants $\alpha\in(0,\frac{1}{2\zeta}]$ and $\delta\in(0,2+\lambda]$,

\begin{eqnarray}\nn
\begin{cases}
&\mathbb{P}[{\MVbar}_t-\MV(X)>\delta] \le 2 \exp(-\frac{\alpha t\delta^2}{(2+\lambda)^2}),\\\label{narnia1}
&\mathbb{P}[{\MVbar}_t-\MV(X)<-\delta] \le 2 \exp(-\frac{\alpha t\delta^2}{(2+\lambda)^2}).
\end{cases}
\end{eqnarray}

\end{lemma}

\section{Lower Bounds}\label{Sec:LB}

\subsection{The Decision Variance and the Decomposition of the Regret}

In this subsection, we derive a compact analytical expression of the regret of any given policy $\pi$. This expression shows a decomposition of regret into two terms. The first term is given by the expected number of times suboptimal actions are chosen. The second term, which is absent in the regret under the expected reward criterion, captures the role of the variance in the actions (due to the mapping from past random observations) in excess mean-variance. This result also shows that the optimal policy $\pi^*$ under known models is an optimal single action policy, a fact that is not obvious as in the risk-neutral case. 

Lemma~\ref{Lemma:RegExp} provides an expression of regret which is used throughout the paper to analyze the performance of the policies. Let $k^*=\argmin_k \MV_k$ (with ties broken arbitrarily), $\Gamma_k=\MV_{k}-\MV_{k^*}$ and $\Delta_k=\mu_k-\mu_{k^*}$.

\begin{lemma}\label{Lemma:RegExp}
The regret of a policy $\pi$ under the measure of total mean-variance of rewards can be expressed as
\begin{eqnarray}\nn
&&\hspace{-2em}\Reg(T) 
= \sum_{k=1}^K \E[\tau_{k,T}]\Gamma_{k} \\\label{RegUB}
&&\hspace{-1.5em}+\sum_{t=1}^T \E\left[\left(\Sumnotkstar (\I[\pi_{t}=k]-\Pr[\pi_{t}=k])\Delta_k\right)^2\right].~~~~
\end{eqnarray}

\end{lemma}
\emph{Proof.} See Appendix A.   

The regret expression given in Lemma~\ref{Lemma:RegExp} shows that $\Reg(T)\ge0$ for any policy $\pi$, and $R_{\pi^*}(T)=0$ for $\pi^*_t=k^*$ (for all $t$), which proves that the optimal single-action policy is the optimal policy under the risk-averse measure.

\subsection{Distribution-Dependent Regret}

The first term in the regret expression given in Lemma~\ref{Lemma:RegExp} 
captures choosing suboptimal actions similar to the risk-neutral setting. Since the second term is always positive, the similar distribution-dependent lower bounds as in the risk-neutral problem hold.
Specifically, under bandit information setting,
an $\Omega(K\log T)$ lower bound for distribution-dependent regret can be established following the similar lines as in the proof of the lower bound results for risk-neutral bandit information setting provided in~\cite{Lai&Robbins85AAM,BR}. Under full information setting, an $\Omega(\log K)$ lower bound for distribution-dependent regret can be established following the similar lines as in the proof of the lower bound results for risk-neutral full information setting provided in~\cite{Mourtada}. 

These results are order optimal since, assuming constant distribution parameters ($\Gamma_k>0$, $\Delta_k$), the distribution-dependent regret incurred due to decision variance is in the same order as the regret incurred due to choosing suboptimal actions. The upper bound results presented in Section~\ref{Sec:Policy} confirm this observation.

Although the two terms in regret show similar distribution-dependent performance, they are different in the dependence to the distribution parameters; specifically $\Delta_k$ and $\Gamma_k$. This different scaling, in comparison to the risk-neutral setting, results in different worst-case regret performance as shown next.

\subsection{Worst-case Regret}

We prove a linear lower bound for risk-averse regret under worst case distribution assignment which is striking in contrast to the sublinear risk-neutral regret. The lower bound is proven under the full information setting. The same lower bound immediately follows under the bandit information setting since the more limited information in the bandit setting cannot improve the performance. In other words, since the bandit information policies are a subset of the full information policies, any lower bound result on the latter also holds for the former.

Our lower bound proof is based on a coupling argument in a problem with $2$ actions. Let $\mathcal{F} = (f_1, f_2$) and $\mathcal{F}' = (f_1, f'_2)$ denote two different distribution models.
Let $f_1 \sim \mathcal{N}(\mu_1,\sigma_1^2)$, a normal distribution with mean $\mu_1=\frac{3}{2}$ and variance $\sigma_1^2=\frac{3}{16}-4\Gamma^2$, for some $\Gamma\in(0,\frac{1}{8})$. Also, let $f_2\sim\mathcal{B}(p)$, a Bernoulli distribution with $p=1/4+2\Gamma$, and $f'_2\sim\mathcal{B}(q)$ a Bernoulli distribution with $q=1/4-2\Gamma$. 
For any action selection policy $\pi$, we prove that, under at least one of the two systems, the number of times the suboptimal action is chosen is high in expectation. 


\begin{lemma}\label{Thm1}
For any policy $\pi$ with full information and any parameter $\Gamma>0$, in the $2$-action problem described above with the number of rounds $T\ge100$,
\begin{eqnarray}\label{Thm1In1}
\{\E_{\F}[\tau_{2,T}]\vee\E_{\F'}[\tau_{1,T}]\}\ge \left\{\frac{0.01}{\Gamma^2} \wedge \frac{T}{2e}\right\}\footnotemark.
\end{eqnarray}
\footnotetext{The notation $\{a\vee b\}$ ($\{a\wedge b\}$) denotes the maximum (minimum) of two real numbers $a$ and $b$.}
\end{lemma}
\emph{Proof.} See Appendix B.

Using Lemma~\ref{Thm1}, we establish a lower bound on the worst case regret performance of any policy $\pi$.  

\begin{theorem}\label{Thm2}
For any action selection policy $\pi$ with full information, there exists a distribution assignment $\F$ to a $2$-action problem where
\begin{eqnarray}\label{LowBou}
\Reg(T) \ge\frac{T}{4e}.
\end{eqnarray}
\end{theorem}

\begin{proof}

The first and the second terms in the regret expression given in Lemma~\ref{Lemma:RegExp}
correspond to the expected value and the variance of choosing suboptimal actions, respectively. We prove that there exists a mapping from any policy $\pi$  to a new policy whose expected number of choosing suboptimal actions gives a lower bound on the total expected variance of $\pi$. This interesting observation together with Lemma~\ref{Thm1} proves the theorem. A detailed proof is given below. 

Let $[T]=\{1,2,\dots,T\}$ denote the set of time instances. For each $S\subseteq [T]$ and any policy $\pi$ in a 2-action problem, we construct a new policy $\pi^{S}$, based on $\pi$, that is obtained by altering the decision of policy $\pi$ on set $S$. In particular,
\begin{eqnarray}
\begin{cases}
&\pi^{S}_t = \pi_t,~~\text{if}~t\not\in S\\
&\pi^{S}_t = 3-\pi_t,~~\text{if}~t\in S.\\
\end{cases}
\end{eqnarray}

In a 2-action problem, let $\Delta=\Delta_k$ where $k\in\{1,2\}$ and $k\neq k^*$. 
In the second term in regret expression given in~\eqref{RegUB}, we have
\begin{eqnarray}\nn
&&\hspace{-5em}\E_\F\left[\left(\sum_{\substack{k=1 \\k\neq k^*}}^K (\I[\pi_{t}=k]-\Pr_\F[\pi_{t}=k])\Delta_k\right)^2\right]\\\nn 
&=& 
\E_\F\left[\left( (\I[\pi_{t}\neq k^*]-\Pr_\F[\pi_{t}\neq k^*])\Delta\right)^2\right]\\\nn
&=&\Pr_\F[\pi_{t}\neq k^*](1-\Pr_\F[\pi_{t}\neq k^*])\Delta^2.
\end{eqnarray}
The first term in the regret expression given in~\eqref{RegUB}, is always positive. Thus
\begin{eqnarray}\label{Seq0}
\Reg(T)\ge \sum_{t=1}^T \Pr_\F[\pi_{t}\neq k^*](1-\Pr_\F[\pi_{t}\neq k^*])\Delta^2.
\end{eqnarray}

For $t\in S$, $\Pr[\pi^S_{t}\neq k^*]= \Pr[\pi_{t}\neq k^*]$ because $\pi_t^S = \pi_t$; and for $t\not\in S$, $\Pr[\pi^S_{t}\neq k^*]=1- \Pr[\pi_{t}\neq k^*]$ because $\pi_t^S = 3-\pi_t$. We thus have, for all $S\subseteq[T]$
\begin{eqnarray}\nn
&&\hspace{-5em}\Pr_\F[\pi^S_{t}\neq k^*](1-\Pr_\F[\pi^S_{t}\neq k^*])\Delta^2 = \\\label{Seq1} &&~~~~~\Pr_\F[\pi_{t}\neq k^*](1-\Pr_\F[\pi_{t}\neq k^*])\Delta^2.
\end{eqnarray}
By construction of $\{\pi^S\}_{S\subseteq[T]}$, there exists a $S_0\subseteq[T]$ that $\Pr_\F[\pi^{S_0}_{t}\neq k^*]\le\frac{1}{2}$ for all $t\in[T]$. For $S_0$, we have
\begin{eqnarray}\nn
&&\hspace{-5em}\sum_{t=1}^T\Pr_\F[\pi^{S_0}_{t}\neq k^*](1-\Pr_\F[\pi^{S_0}_{t}=2])\Delta^2\\\label{Sineq1}
&&~~~\ge\frac{1}{2}\sum_{t=1}^T \Pr_\F[\pi^{S_0}_{t}\neq k^*]\Delta^2.
\end{eqnarray}

From Lemma~\ref{Thm1}, there exists a distribution $\F$ for a 2-action problem where 
\begin{eqnarray}\label{Sineq2}
\sum_{t=1}^T\Pr_\F[\pi^{S_0}_t\neq k^*] \ge \{\frac{0.01}{\Gamma^2}\wedge\frac{T}{2e}\}.
\end{eqnarray}

Thus, combining~\eqref{Seq0},~\eqref{Seq1},~\eqref{Sineq1} and~\eqref{Sineq2}, there exists a distribution model $\F$ for the 2-action problem where
\begin{eqnarray}\nn
\Reg(T) &\ge& \sum_{t=1}^T \Pr_\F[\pi_{t}\neq k^*](1-\Pr_\F[\pi_{t}\neq k^*])\Delta^2\\\nn 
&=&\sum_{t=1}^T \Pr_\F[\pi^{S_0}_{t}\neq k^*](1-\Pr_\F[\pi^{S_0}_{t}\neq k^*])\Delta^2\\\nn
&\ge&\frac{1}{2}\sum_{t=1}^T \Pr_\F[\pi^{S_0}_{t}\neq k^*]\Delta^2\\\nn
&\ge& \{\frac{0.005}{\Gamma^2}\wedge\frac{T}{4e}\}\Delta^2.
\end{eqnarray}

Choosing the worst case $\Gamma=\sqrt{\frac{0.02e}{T}}$, and for $\Delta=1$, we have
\begin{eqnarray}\nn
\Reg(T) \ge \frac{T}{4e},
\end{eqnarray}
which completes the proof. 

\end{proof}

We point out that considering only $2$ actions does not limit the extension of the lower bound result to the problems with more than $2$ actions. Specifically the same lower bound with the same proof holds for a problem with $K>2$ actions where the actions $k=3,4,\dots, K$ are suboptimal in both $\F$ and $\F'$.
Our lower bound proof however lacks the dependency on the number of actions. Nevertheless, notice that a linear lower bound on regret shows the impossibility of converging to the performance of the optimal policy regardless of dependency on~$K$.  

The linear lower bound on the regret holds irrespective to the value of $\lambda$. The reason is that $\lambda$ appears only in the first term in the regret corresponding to choosing suboptimal actions. The second term in the regret which corresponds to the decision variance (and has a dominant effect on the worst case regret lower bound) is independent of $\lambda$. 

\section{Risk-Averse Policies}\label{Sec:Policy}

In this section, we introduce and analyze the performance of several risk-averse policies under both bandit and full information settings.

\subsection{The Bandit Setting}

Under bandit information setting we analyze the performance of Mean-Variance Lower Confidence Bound (\MVL~) policy and Confidence Bounds based Action Elimination (\CBE) policy.

\MVL~ is a modification of the classic UCB policy first introduced in~\cite{Auer&etal02ML} for risk-neutral bandits                                                                                                                                                                                                                                                                                                                                                                                                                                                                                                                                                                                                                                                                                                                                                                                                                                                                                                                                                                                                                                                                                                                                                                                                                                                                                                                                                                                                                                                                              and then adopted for risk-averse bandits in~\cite{Sani, Vakili2015}. At each time $t$, \MVL~ chooses the action with the smallest lower confidence bound on mean-variance:
\begin{eqnarray}
\pi^{\MVL}_t=\argmin_{k}\MVbar_{k,t}-\sqrt{\frac{c\log t}{\tau_{k,t}}},
\end{eqnarray}
where $c$ is a constant that depends on the distribution class parameter $\alpha$ (as specified in Lemma~\ref{Lemma:CHB}).
\begin{algorithm}[h]
\begin{algorithmic}[1]
    \State Initialization: $T\in \N$, $[K]$, $\MVbar_{k,1}=0$, $\tau_{k,1}=0$, for all $k\in [K]$.
    \vspace{.5em}
    \For{$t=1,2,\dots,T$}
    \State Play $\pi^{\MVL}_t=\argmin_{k}\MVbar_{k,t}-\sqrt{\frac{c\log t}{\tau_{k,t}}}$
    \State Update $\MVbar_{k,t}$ and $\tau_{k,t}$.
    \EndFor
    

\end{algorithmic}\label{AlgMVP0}
 \caption{\MVL~ Policy.}
\end{algorithm}

\begin{theorem}\label{Thm3}
When there is a positive gap in the mean-variances of the best and the second best actions, for $c\ge\frac{3(2+\lambda)^2}{\alpha}$, the regret of \MVL~ policy satisfies{\footnote{$\alpha$ is the distribution class parameter specified in concentration inequalities in Lemma~\ref{Lemma:CHB}.}}
\begin{eqnarray}\nn
&&\hspace{-3em}R^{\pi^{\MVL}}(T)\le
\\
&&\hspace{-3em}\Sumnotkstar\left(\frac{4c\log T}{\Gamma_k^2}+5\wedge T\right)\bigg(\Gamma_k+\frac{(K-1)\Delta_k^2}{4}\bigg).
\end{eqnarray}
\end{theorem}
\emph{Proof.} See Appendix C. 

Theorem~\ref{Thm3} shows a logarithmic upper bound on the distribution-dependent regret of~\MVL~for easy problems where there is a positive gap $\Gamma=\min_{k} \{\Gamma_k:\Gamma_k>0\}$ 
in the mean variances of the best and the second best actions. Notice that when $\Gamma\rightarrow 0$ the upper bound grows to be linear in $T$.

The \CBE~policy is a modification of Improved UCB introduced in~\cite{UCBRev} which proceeds in steps $n=0,1,2,\dots$. At each step $n$, a set of actions $\K_n$, initialized at $\K_0=[K]$, are chosen, each $u_n=\l \frac{C\log T}{\Gahat_n^{2}} \r$ times where $\Gahat_n=\Gahat_02^{-n}$ is initialized at $\Gahat_0>0$ and $C>0$ is a constant that depends only on the distribution class parameter $\alpha$. At each step, a number of actions are potentially removed from $\K_n$ based on upper and lower confidence bounds on their mean-variance, respectively, in the from of $\MVbar_{k}^{(n)}+\frac{\hat{\Gamma}_n}{4}$ and $\MVbar_{j}^{(n)}-\frac{\hat{\Gamma}_n}{4}$, where $\MVbar_{k}^{(n)}$ is the sample mean-variance obtained from the $u_n$ observations at step $n$. 
If the lower confidence bound of action $k$ is bigger than the minimum of the upper confidence bounds of all other remaining actions, action $k$ is removed  $\K_{n+1}=\K_n\setminus\{k\}$; see lines 6-10 in Algorithm~\ref{AlgMVP}.

Let $n_k=\min\{n:\Gahat_n\le\Gamma_k\}$ and $n_{\max}$ be the number of steps taken in \CBE. Let $\Delta_{\max}=\max_{k\in[K]\setminus*}|\Delta_k|$. 
\begin{theorem}\label{Thm4}
The risk-averse regret performance of \CBE~policy, for $C\ge\frac{64}{\alpha}$, satisfies
{\small{\begin{eqnarray}\nn
&&\hspace{-3em}R^{\pi^{\CBE}}(T)\\\nn
&&\hspace{-3em}\le \Sumnotkstar \left(\frac{\frac{4C}{3}\log T}{\Gamma_k^2}+\log_2\left(\frac{1}{\Gamma_k}\right) +\frac{K\log_2T+2}{T^3}\wedge T\right) \Gamma_k\\\nn
&+&
\frac{1}{2}\log_2T\Delta_{max}^2\Sumnotkstar\bigg( \bigg(\frac{C\log T}{\Gamma_k^2}+1\bigg)\I[n_k\le n_{\max}]\\\nn
&+&
\bigg(\frac{\frac{C}{4}\log T}{\Gamma_k^2}+1\bigg)\I[n_k-1\le n_{\max}]\bigg)\\
&+&\left(\frac{K\log_2 T+2}{T^4}+\frac{K\log_2 T}{T}\right)\left(\frac{(K-1)^2T\Delta_{\max}^2}{4}\right).~~~~~
\end{eqnarray}}}
\end{theorem}
\emph{Proof.} See Appendix D. 

Theorem~\ref{Thm3} shows a logarithmic upper bound on the distribution-dependent regret of~\CBE. The worst case regret of \CBE~corresponds to the cases where there exists a $k$ with $\Gamma_k=\Theta (\frac{1}{\sqrt T})$. Unlike \MVL, \CBE~recovers the sublinear regret for the smaller orders of $\Gamma_k$. Specifically, with equally good actions in terms of their mean variance, \CBE~has a $0$ regret which is not the case with \MVL~, as it is shown in the simulations section. 
 \begin{algorithm}[h]
\begin{algorithmic}[1]
    \State Initialization: $\Gahat_0=1$, $n=0$, $T\in \N$, $\K_0=[K]$.
    \vspace{.5em}
\While{time is left}
        \State $\K_{n+1}=\K_n$
        \State $u_n=\l \frac{C\log T}{\Gahat_n^{2}} \r$. 
        \State Choose each action $k\in \K_n$ for $u_n$ times. 
        \For{$k\in \K_n$}
            \If{
            $
            \MVbar_{k}^{(n)}- \frac{\Gahat_n}{4}> \min_{j\in\K_n}\MVbar_{j}^{(n)} +\frac{\Gahat_n}{4}
            $}
                \State Remove action $k$: $\K_{n+1} \leftarrow \K_{n+1}\setminus\{k\}$.
            \EndIf
        \EndFor
\State n=n+1
\State $\Gahat_{n+1}=\frac{\Gahat_n}{2}$
\EndWhile
\vspace{.5em}
\end{algorithmic}
\caption{\CBE~Policy.}\label{AlgMVP}
\end{algorithm}

\subsection{The Full Information Setting}

Full information from actions renders the need for bandit exploration obsolete. The simple Follow the Leader (FL) policy is a common policy in the risk-neutral problem. A straightforward modification of FL for risk-averse problem gives us the policy
\begin{eqnarray}
\pi_t^{\MVFL} = \argmin \MVbar_{k,t}. 
\end{eqnarray}

\begin{theorem}
The risk-averse regret performance of \MVFL~ satisfies
\begin{eqnarray}\nn
&&\hspace{-3em}R^{\pi^{MV-FL}}(T)\le
\\
&&\hspace{-3em}\left(\frac{4}{\alpha\Gamma^2}(\log K+1)+1\wedge T\right) \left(\Gamma+\frac{(K-1)\Delta_{\max}^2}{4}  \right).~~~
\end{eqnarray}
\end{theorem}

Parallel to the bandit information setting, a more structured policy based on action elimination is expected to offer a better risk-averse regret. Specifically, the same CB-AE policy can be used in the full information setting with two changes: first, the sample mean-variance is calculated based on full information available at each step, second, leveraging the full information the value of $u_n$ is reduced to $u_n=\l \frac{C\log T}{|\K_n|\Gahat_n^{2}} \r$.

\section{Simulations} \label{sec.sims}

In this section, we provide simulation results on the performance of \MVL, \CBE, and \MVFL.  
We compare the performance of \MVL~and \CBE~in Figure~\ref{Sim1}. As it is expected, \CBE~shows a better regret performance in the simulations in comparison to \MVL. The reason is that \CBE,~by fixing the action elimination structure, reduces the variance in the decisions. 
While both policies show a linear worst case regret performance, \MVL~ has a linear regret performance for all the settings where there exists a $k\neq k^*$ with $\Gamma_k=\Or(\frac{1}{\sqrt T})$ and $\Delta_k>>0$. 
On the other hand, \CBE,~as it can be seen from the upper bound in~Theorem~\ref{Thm4}, has a linear regret for the particular case of $\Gamma_k=\Theta(\frac{1}{\sqrt T})$ and $\Delta_k>>0$. Specifically, the \CBE~policy recovers the sublinear regret for the smaller values of $\Gamma_k$ (when $\Gamma_k\rightarrow 0$).

Figure~\ref{Sim.fullfeedback} shows the comparison of \MVFL~ and \CBE~under full feedback setting. 
While for easy models with relatively large $\Gamma$, \MVFL~ works well and has a sublinear regret, with $\Gamma\rightarrow 0$ the regret grows to linear with time. \CBE~, on the other hand, recovers the sublinear regret when $\Gamma\rightarrow 0$. 

\begin{figure} [h]
\centering
\mbox {
\subfigure[$\Gamma=0.50$]{\includegraphics[width=0.22\columnwidth]{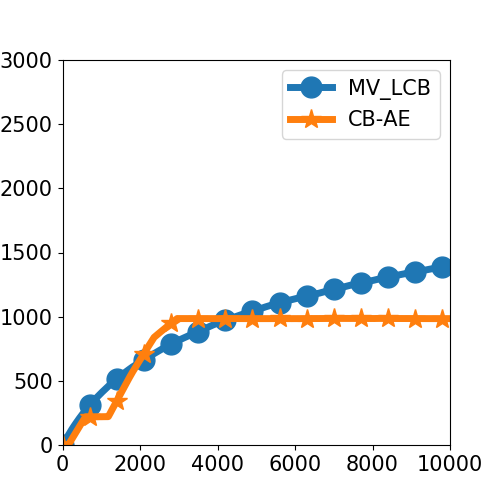}}
\subfigure[$\Gamma=0.20$]{\includegraphics[width=0.22\columnwidth]{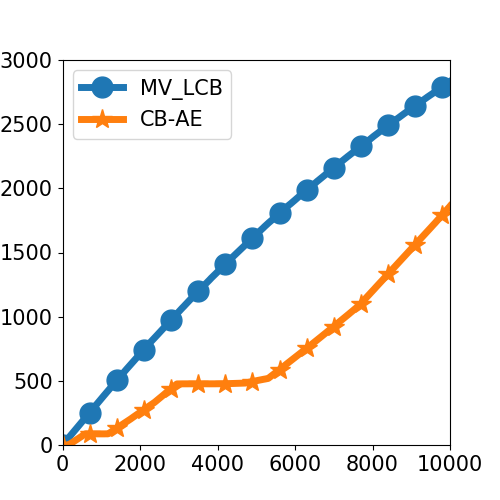}}
\subfigure[$\Gamma=0.10$]{\includegraphics[width=0.22\columnwidth]{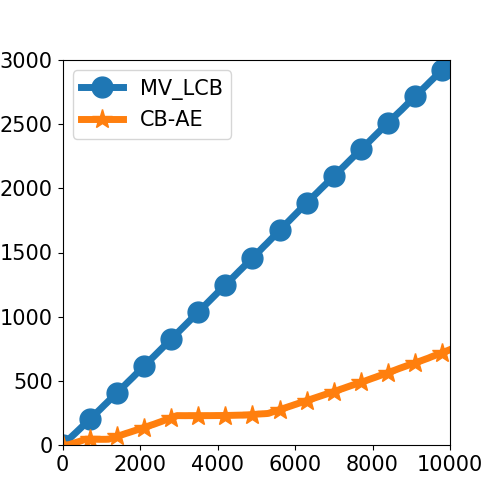}}
} \mbox {
\subfigure[$\Gamma=0.05$]{\includegraphics[width=0.22\columnwidth]{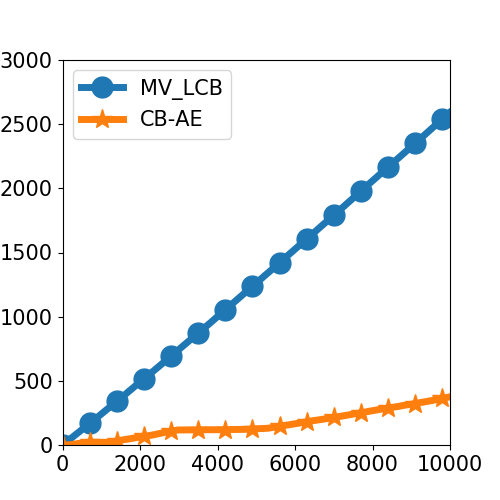}}
\subfigure[$\Gamma=0.01$]{\includegraphics[width=0.22\columnwidth]{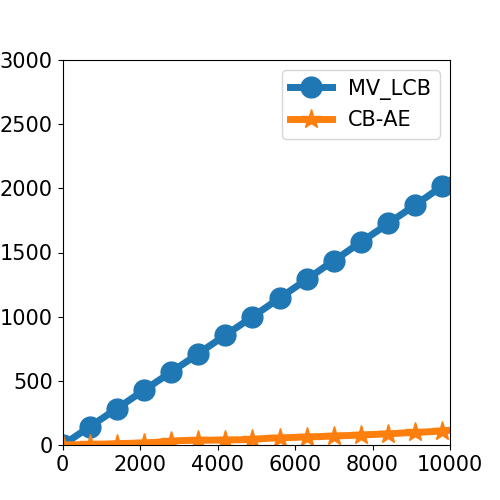}}
\subfigure[$\Gamma=0.00$]{\includegraphics[width=0.22\columnwidth]{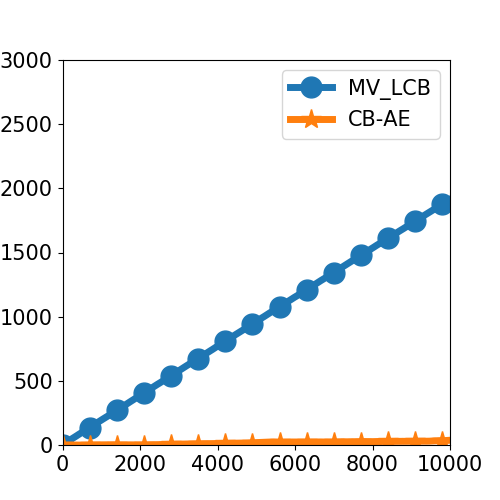}}
}
\caption{Comparison of the performance of \MVL~and~\CBE~in terms of their regret over time for different values of $\Gamma$.}\label{Sim1}
\end{figure}

In this simulation, $K=4$ actions are Binomially distributed with mean $\mu_*=1$ and variance $\sigma_*^2 = 1$ for the optimal action. For other actions we choose $\mu_k=2$ and vary the variance over the set $\{2.5, 2.2, 2.1, 2.05, 2.01, 2.0\}$ simulating different $\Gamma$ values. The time horizon is varied from $T=1$ to $T=10000$ and the regret curves are average performance over $1000$ Monte Carlo runs. The parameters for \MVL~and \CBE~are $c=1$, $\Gamma_0=1$, and $C=16$.

\begin{figure}[h] 
\centering
\mbox {
\subfigure[$\Gamma=0.50$]{\includegraphics[width=0.22\columnwidth]{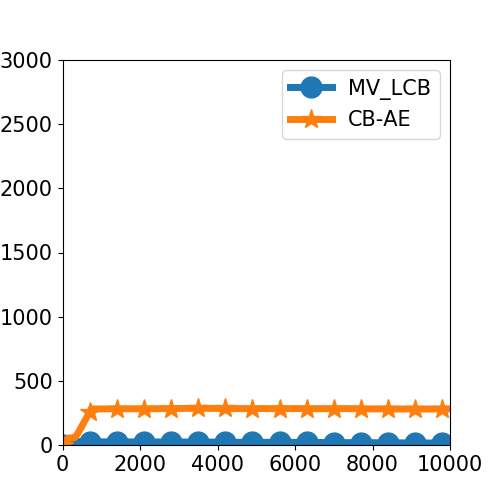}}
\subfigure[$\Gamma=0.20$]{\includegraphics[width=0.22\columnwidth]{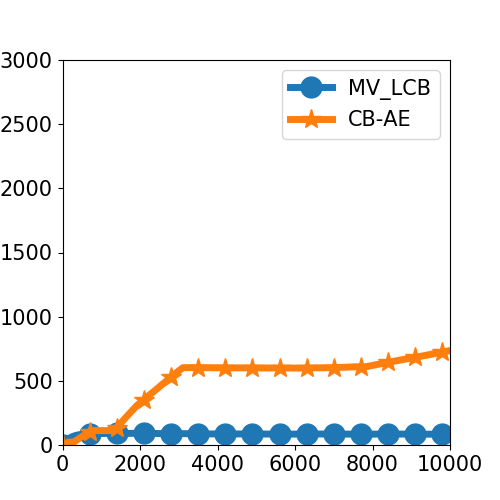}}
\subfigure[$\Gamma=0.10$]{\includegraphics[width=0.22\columnwidth]{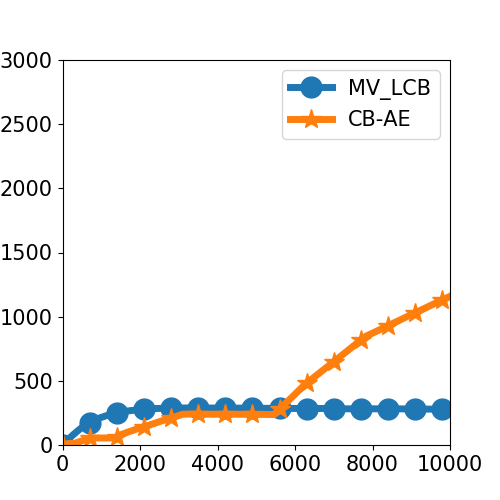}}
} \mbox {
\subfigure[$\Gamma=0.05$]{\includegraphics[width=0.22\columnwidth]{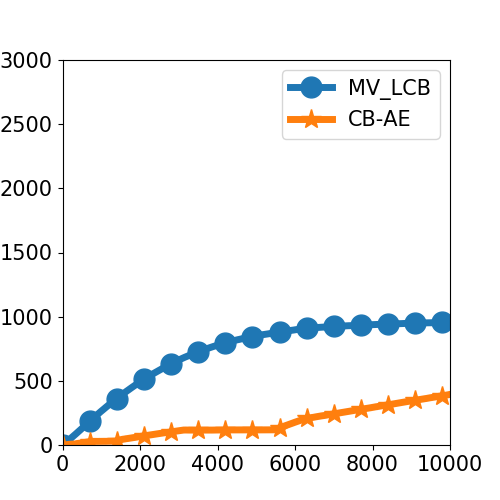}}
\subfigure[$\Gamma=0.01$]{\includegraphics[width=0.22\columnwidth]{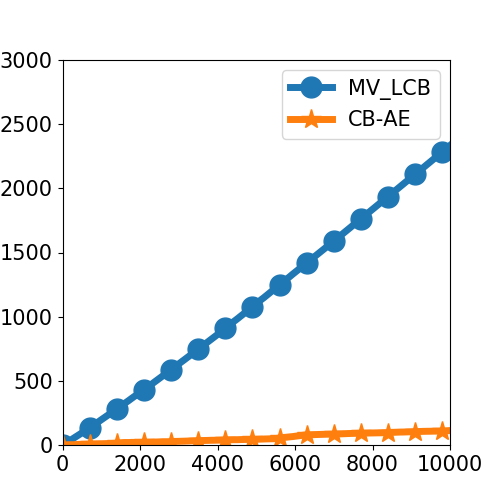}}
\subfigure[$\Gamma=0.00$]{\includegraphics[width=0.22\columnwidth]{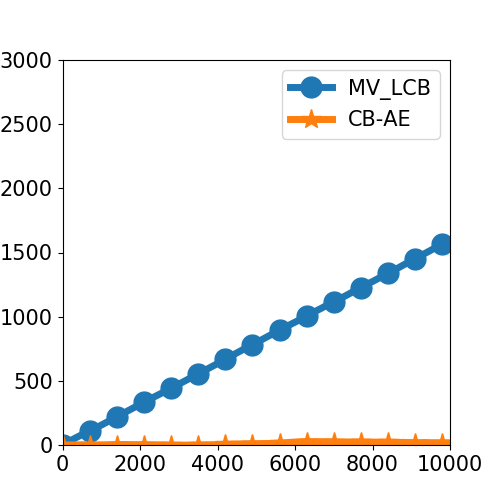}}
}
\caption{Comparison of the performance of \MVFL~and~\CBE~in terms of their regret over time for different values of $\Gamma$.}\label{Sim.fullfeedback}
\end{figure}

\section{Conclusion}
In this paper, we studied online learning problems under a mean-variance measure. We showed that a dominant term in risk-averse regret comes from the variance in the decisions. We established fundamental limits on learning policies; while a logarithmic distribution-dependent regret is achievable by UCB and FL type policies, similar to the risk-neutral settings, an $\Omega(T)$ worst case regret is inevitable in contrast to the $\Omega(\sqrt T)$ counterpart lower bound in the risk-neutral setting.

\section*{Appendix A}
\begin{proof}[Proof of Lemma~\ref{Lemma:RegExp}]
We analyze the mean and the variance of the observed reward at time $t$ under policy $\pi$. For the $\E[X_{\pi_t,t}]$ we have:
\begin{eqnarray}\nn
\E[X_{\pi_t,t}]
&=& \E[\sum_{k=1}^K\I[\pi_{t}=k]X_{k,t}]\\\label{Lin1}
&=&\sum_{k=1}^K\E[\I[\pi_{t}=k]X_{k,t}]\\\label{Con1}
&=&\sum_{k=1}^K\E\bigg[\E\bigg[\I[\pi_{t}=k]X_{k,t}\bigg|\I[\pi_{t}=k]\bigg]\bigg] \\\nn
&=&\sum_{k=1}^K\E\bigg[\I[\pi_{t}=k]\E\bigg[X_{k,t}\bigg|\I[\pi_{t}=k]\bigg]\bigg]\\\nn
&=&\sum_{k=1}^K\E\bigg[\I[\pi_{t}=k]\mu_k\bigg]\\\label{MN1}
&=&\sum_{k=1}^K\Pr[\pi_t=k]\mu_k.
\end{eqnarray}
Equation~\eqref{Lin1} comes from the linearity of the expectation and equation~\eqref{Con1} is a result of the property of the conditional expectation that for two random variables $Y$ and $Z$: $\E[YZ]=\E[\E[YZ|Z]]$.  

For the variance of $X_{\pi_t,t}$, we have
{\small{\begin{eqnarray}\nn
\E\bigg[\bigg(X_{\pi_{t},t}-\E[X_{\pi_{t},t}]\bigg)^2\bigg]
&=&
\E\bigg[\bigg(\sum_{k=1}^K\I[\pi_{t}=k]X_{k,t}-\E[\sum_{k=1}^K\I[\pi_{t}=k]X_{k,t}]\bigg)^2\bigg] \\\nn
&=&\E\bigg[\bigg( \sum_{k=1}^K\I[\pi_{t}=k]X_k-\sum_{k=1}^K\I[\pi_{t}=k]\mu_k\\\nn
&&~~~+ \sum_{k=1}^K\I[\pi_{t}=k]\mu_k - \sum_{k=1}^K \Pr[\pi_{t}=k]\mu_k \bigg)^2\bigg] \\ \nn
&=&\E\bigg[\bigg( \sum_{k=1}^K\I[\pi_{t}=k](X_k-\mu_k)\\\nn
&&~~~+ 
\sum_{k=1}^K (\I[\pi_{t}=k]-\Pr[\pi_{t}=k])\mu_k \bigg)^2\bigg] \\ \nn
&=&
\underbrace{\E\left[\left( \sum_{k=1}^K\I[\pi_{t}=k](X_k-\mu_k) \right)^2\right]}_{\text{The first term}} \\\nn
&&~~~+ 
\underbrace{\E\left[\left(\sum_{k=1}^K (\I[\pi_{t}=k]-\Pr[\pi_{t}=k])\mu_k \right)^2\right]}_{\text{The second term}} \\\label{X11}
&&~~~+ \tiny{
\underbrace{2 \E\left[\left( \sum_{k=1}^K\I[\pi_{t}=k](X_k-\mu_k) \right)\left(\sum_{k=1}^K (\I[\pi_{t}=k]-\Pr[\pi_{t}=k])\mu_k \right)\right]}_{\text{The third term}}}.~~~
\end{eqnarray}}}
We analyze the three term in~\eqref{X11} separately. 

The first term: 
{\small{\begin{eqnarray}\nn
&&\hspace{-9em}{\E\left[\left( \sum_{k=1}^K\I[\pi_{t}=k](X_{k,t}-\mu_k) \right)^2\right]}\\\nn
&&\hspace{-4em}= \E\left[\left( \sum_{j=1}^K\I[\pi_{t}=j](X_{j,t}-\mu_j) \right)\left( \sum_{k=1}^K\I[\pi_{t}=k](X_{k,t}-\mu_k) \right)\right]\\\nn
&&\hspace{-4em}= \E\left[ \sum_{j=1}^K\sum_{k=1}^K\I[\pi_{t}=j]\I[\pi_{t}=k](X_{j,t}-\mu_j) (X_{k,t}-\mu_k)\right]\\\nn
&&= \sum_{j=1}^K\sum_{k=1}^K\E\bigg[ \I[\pi_{t}=j]\I[\pi_{t}=k](X_{j,t}-\mu_j) (X_{k,t}-\mu_k)\bigg]\\\nn
&&\hspace{-4em}=\sum_{k=1}^K\E\bigg[ \I[\pi_{t}=k] (X_{k,t}-\mu_k)^2\bigg]\\\nn
&&\hspace{-4em}~~~+\sum_{j=1}^K\sum_{\substack{k=1 \\ k\neq j}}^K\E\bigg[ \I[\pi_{t}=j]\I[\pi_{t}=k](X_{j,t}-\mu_j) (X_{k,t}-\mu_k)\bigg]\\\nn
\\\label{Tr1}
&&\hspace{-4em}=\sum_{k=1}^K \Pr[\pi_{t}=k]\sigma^2_k. 
\end{eqnarray}}}
The last equality is proven similar to \eqref{MN1}.

The second term:
{\small{\begin{eqnarray}\nn
&&\hspace{-7em}\E\left[\left(\sum_{k=1}^K (\I[\pi_{t}=k]-\Pr[\pi_{t}=k])\mu_k \right)^2\right]\\\nn
&&\hspace{-3em}=\E\Bigg[\bigg(\sum_{\substack{k=1 \\k\neq k^*}}^K (\I[\pi_{t}=k]-\Pr[\pi_{t}=k])\mu_k \\\nn
&&\hspace{-3em}~~~+ (\I[\pi_{t}=k^*]-\Pr[\pi_{t}=k^*])\mu_{k^*}\bigg)^2\bigg]\\\label{su1}
&&\hspace{-3em}=\E\bigg[\bigg(\sum_{\substack{k=1 \\k\neq k^*}}^K (\I[\pi_{t}=k]-\Pr[\pi_{t}=k])\mu_k \\\nn
&&\hspace{-3em}~~~+
\bigg(1-
\sum_{\substack{k=1 \\k\neq k^*}}^K \I[\pi_{t}=k]-(1-\sum_{\substack{k=1 \\k\neq k^*}}^K\Pr[\pi_{t}=k])\bigg)\mu_{k^*}\bigg)^2\bigg]~~~~\\\label{Tr2}
&&\hspace{-3em}=\E\left[\left(\sum_{\substack{k=1 \\k\neq k^*}}^K (\I[\pi_{t}=k]-\Pr[\pi_{t}=k])\Delta_k\right)^2\right].
\end{eqnarray}}}
The equation~\eqref{su1} holds because $\sum_{k=1}^K\I[\pi_t=k]=1$ and $\sum_{k=1}^K\Pr[\pi_t=k]=1$.

The third term: 
{\small{\begin{eqnarray}\nn
&&\hspace{-3em}\tiny{\E\left[\left( \sum_{k=1}^K\I[\pi_{t}=k](X_k-\mu_k) \right)\left(\sum_{k=1}^K (\I[\pi_{t}=k]-\Pr[\pi_{t}=k])\mu_k \right)\right] }\\ \nn
&=&\E\bigg[\E\bigg[\bigg( \sum_{k=1}^K\I[\pi_{t}=k](X_k-\mu_k) \bigg)\bigg(\\\nn
&&\sum_{k=1}^K (\I[\pi_{t}=k]-\Pr[\pi_{t}=k])\mu_k \bigg)\bigg|\I[\pi_{t}=k] \bigg]\bigg]\\\label{Tr3}
&=& 0.
\end{eqnarray}}}

Combining~\eqref{MN1},~\eqref{X11},~\eqref{Tr1},~\eqref{Tr2},~\eqref{Tr3}, we have 
\begin{eqnarray}\nn
\MV(X_{\pi_t,t})
&=& \sum_{k=1}^K \Pr[\pi_t=k]\MV_k \\\nn
&+&  \E\left[\left(\sum_{\substack{k=1 \\k\neq k^*}}^K (\I[\pi_{t}=k]-\Pr[\pi_{t}=k])\Delta_k\right)^2\right].
\end{eqnarray}
Summing up the mean variance of observations at each time and subtracting that of the optimal single arm strategy we arrive at~\eqref{RegUB}.
\end{proof}

\section*{Appendix B}
\begin{proof}[Proof of Lemma~\ref{Thm1}]

The following lemma is used in establishing the lower bound for worst case regret under risk-averse setting. 
\begin{lemma}\label{KLLemma}
Let $\nu$ and and $\nu'$ be two probability distributions supported on some set $\X$ with $\nu'$ absolutely continuous with respect to $\nu$. For any measurable function $\phi: \X\rightarrow \{0,1\}$, we have
\begin{eqnarray}
\Pr_{\nu}(\phi(X)=1) + \Pr_{\nu'}(\phi(X)=0)\ge \frac{1}{2}\exp(-\KL(\nu,\nu')).
\end{eqnarray}
\end{lemma}
Notation $\Pr_{\nu}(.)$ denotes the probability measure with respect to $\nu$ and notation $\KL(\nu,\nu')$ denotes the Kullback-Leibler divergence between $\nu$ and $\nu'$ defined as $\KL(\nu,\nu')=\E_{\nu}[\log\frac{d\nu}{d\nu'}]$. Lemma~\ref{KLLemma} was used in~\cite{BR} to establish a lower bound on the risk-neutral bandit regret with side information. 

For the KL divergence between $f_2$ and $f'_2$, we have
\begin{eqnarray}\nn
&&\hspace{-3em}\KL(f_2,f_2') = p\log\frac{p}{q} + (1-p)\log\frac{1-p}{1-q}\\ \nn
&=& -(\frac{1}{4}+2\Gamma)\log\frac{\frac{1}{4}-2\Gamma}{\frac{1}{4}+2\Gamma} - (\frac{3}{4}-2\Gamma)\log\frac{\frac{3}{4}+2\Gamma}{\frac{3}{4}-2\Gamma}\\\nn
&=& -(\frac{1}{4}+2\Gamma)\log(1-\frac{4\Gamma}{\frac{1}{4}+2\Gamma}) - (\frac{3}{4}-2\Gamma)\log(1+\frac{4\Gamma}{\frac{3}{4}-2\Gamma})\\\nn
&\le& -(\frac{1}{4}+2\Gamma)\left( -\frac{4\Gamma}{\frac{1}{4}+2\Gamma} -\frac{1}{2}(\frac{4\Gamma}{\frac{1}{4}+2\Gamma})^2 
-\frac{1}{3}(\frac{4\Gamma}{\frac{1}{4}+2\Gamma})^3 \right)\\\label{TTE1}
&&~~~~~- (\frac{3}{4}-2\Gamma)\left(
\frac{4\Gamma}{\frac{3}{4}-2\Gamma}
+\frac{1}{2}(\frac{4\Gamma}{\frac{3}{4}-2\Gamma})^2
+\frac{1}{3(1+\frac{1}{8})^3}(\frac{4\Gamma}{\frac{3}{4}-2\Gamma})^3
\right)\\\nn
&=&\Gamma^2\left(
\frac{8}{\frac{1}{4}+2\Gamma}
+\frac{64\Gamma }{3(\frac{1}{4}+2\Gamma)^2 }
-\frac{8}{\frac{3}{4}-2\Gamma}
-\frac{64\Gamma }{3(1+\frac{1}{8})^3(\frac{3}{4}-2\Gamma)^2 }
\right) \\\label{KLB1}
&\le&22\Gamma^2.
\end{eqnarray}
Inequality~\eqref{TTE1} is obtained based on truncated Taylor expansion of $\log(1+x)$ for $x\in(-1,1)$ and the last inequality holds for all $\Gamma\in(0,\frac{1}{8})$.

Let $f_k^{(t)}(x_{k,1},x_{k,2},\dots,x_{k,t})=\Pi_{s=1}^t f_k(x_{k,s})$ denote the joint distribution of the samples drawn from $f_k$. 
\begin{eqnarray}\nn
\{\E_{\F}[\tau_{2,T}]\vee\E_{\F'}[\tau_{1,T}]\}&\ge& \frac{1}{2}\left(\E_{\F}[\tau_{2,T}]+\E_{\F'}[\tau_{1,T}\right)\\\nn
&=&\frac{1}{2}\sum_{t=1}^T\left(\Pr_{\F}[\I[\pi_t=2]]+\Pr_{\F'}[\I[\pi_t=1]]\right)\\ \label{KLU1}
&\ge& \frac{1}{2}\sum_{t=1}^T\exp(-\KL(f_2^{(t)},f_2'^{(t)}))\\\label{X2}
&=&\frac{1}{2}\sum_{t=1}^T\exp(-\sum_{s=1}^t\KL(f_2,f_2'))\\\label{X3}
&\ge&\frac{1}{2}\sum_{t=1}^T\exp(-22t\Gamma^2). 
\end{eqnarray}
Inequality~\eqref{KLU1} is obtained by Lemma~\ref{KLLemma}. Inequality~\eqref{X2} is based on the assumption of i.i.d. samples for each arm over time, and~\eqref{X3} is obtained by replacing the upper bound on the $\KL(f_2,f_2')$ from~\eqref{KLB1}.
To derive the desired lower bound in~\eqref{Thm1In1} we consider 2 cases for $\Gamma$ as follows.

\paragraph*{Case 1} If $\Gamma\le\frac{1}{\sqrt{22T}}$, then
\begin{eqnarray}\label{Ca1}
\frac{1}{2}\sum_{t=1}^T\exp(-22t\Gamma^2)\ge \frac{1}{2e}T.
\end{eqnarray}

\paragraph*{Case 2} If $\Gamma>\frac{1}{\sqrt{22T}}$, then
\begin{eqnarray}\nn
\frac{1}{2}\sum_{t=1}^T\exp(-22t\Gamma^2)
&\ge&\frac{1}{2}\int_{x=1}^T\exp(-22x\Gamma^2)dx\\\nn
&=&\frac{1}{44\Gamma^2}(\exp(-22\Gamma^2)-\exp(-22T\Gamma^2))\\\nn
&=&\frac{1}{44\Gamma^2}\exp(-22\Gamma^2)\left(1-\exp(-22(T-1)\Gamma^2)\right)\\\nn
&\ge&\frac{\exp(-\frac{22}{64})}{44\Gamma^2}(1-\exp(-\frac{T-1}{T}))\\\label{X6}
&\ge&\frac{\exp(-\frac{22}{64})}{44\Gamma^2}(1-\exp(-\frac{99}{100}))\\\label{Ca2}
&\ge&\frac{0.01}{\Gamma^2}.
\end{eqnarray}
Inequality~\eqref{X6} holds for $T\ge 100$. 

Combining~\eqref{X3},~\eqref{Ca1} and~\eqref{Ca2}, we arrive at the theorem. 
\end{proof}


\section*{Appendix C}
\begin{proof}[Proof of Theorem~\ref{Thm3}]
From the regret expression given in~\eqref{RegUB}, we have
\begin{eqnarray}\nn
\Reg(T) 
&=& \sum_{k=1}^K \E[\tau_{k,T}]\Gamma_{k} 
+\sum_{t=1}^T \E\left[\left(\Sumnotkstar (\I[\pi_{t}=k]-\Pr[\pi_{t}=k])\Delta_k\right)^2\right]\\\nn
&\le& \sum_{k=1}^K \E[\tau_{k,T}]\Gamma_{k} 
+ (K-1)\sum_{t=1}^T \Sumnotkstar\E\left[\left( \I[\pi_{t}=k]-\Pr[\pi_{t}=k]\right)^2\right]\Delta_k^2\\\nn
&=&\sum_{k=1}^K \E[\tau_{k,T}]\Gamma_{k} 
+ (K-1)\Sumnotkstar\sum_{t=1}^T  \Pr[\pi_{t}=k](1-\Pr[\pi_{t}=k])\Delta_k^2\\\label{SO}
&=&\sum_{k=1}^K \E[\tau_{k,T}]\Gamma_{k} 
+ (K-1)\Sumnotkstar\sum_{t=1}^T  \{\Pr[\pi_{t}=k]\wedge\frac{1}{4}\}\Delta_k^2.
\end{eqnarray}

Following the similar line in the analysis of the performance of UCB in~\cite{Auer&etal02ML} and mean-variance UCB in~\cite{Vakili2015} let $b_k=\frac{4c\log T}{\Gamma_k^2}$. 
We have
\begin{eqnarray}\nn
&&\hspace{-3em}\MVbar_{k,t}-\sqrt{\frac{c\log t}{\tau_{k,t}}} - (\MVbar_{*,t}-\sqrt{\frac{c\log t}{\tau_{*,t}}})\\\nn
&=&(\MVbar_{k,t}+\sqrt{\frac{c\log t}{\tau_{k,t}}}-\MV_k)
-(\MVbar_{*,t}-\sqrt{\frac{c\log t}{\tau_{*,t}}}-\MV_*)\\\label{Au1}
&+& (\MV_k-\MV_*-2\sqrt{\frac{c\log t}{\tau_{*,t}}} )
\end{eqnarray}
For $\tau_{k,t}\ge b_k$, the third term in~\eqref{Au1} is positive. Thus, when $\tau_{k,t}\ge b_k$,
\begin{eqnarray}\nn
\Pr[\pi_t=k] &=& \Pr[\MVbar_{k,t}-\sqrt{\frac{c\log t}{\tau_{k,t}}} - (\MVbar_{*,t}-\sqrt{\frac{c\log t}{\tau_{*,t}}})\le 0] \\\nn
&\le& \Pr[\MVbar_{k,t}+\sqrt{\frac{c\log t}{\tau_{k,t}}}-\MV_k\le 0]
+
\Pr[\MVbar_{*,t}-\sqrt{\frac{c\log t}{\tau_{*,t}}}-\MV_*\ge 0] \\\nn
&\le& 4\exp(-\frac{\alpha c \log t}{(2+\rho)^2})\\\nn
&\le& 4t^{-3}. 
\end{eqnarray}
Where the last inequality is obtained by Lemma~\ref{Lemma:CHB}.
We thus have
\begin{eqnarray}\nn
\E[\tau_{k,T}]&\le& b_k + \sum_{t=b_k+1}^T 4t^{-3}\\\label{FTT1}
&\le&
\frac{4c\log T}{\Gamma_k^2}+5
\end{eqnarray}
In the second term in~\eqref{SO}, we have
\begin{eqnarray}\nn
\sum_{t=1}^T  \{\Pr[\pi_{t}=k]\wedge\frac{1}{4}\} &\le& \frac{1}{4}b_k + \sum_{t=b_k+1}^T 4t^{-3}\\\label{STT}
&\le&
\frac{c\log T}{\Gamma_k^2}+5
\end{eqnarray}
Combining~\eqref{SO},~\eqref{FTT1}, and~\eqref{STT}, we arrive at the theorem. 
\end{proof}

\section*{Appendix D}
\begin{proof}[Proof of Theorem~\ref{Thm4}]

To analyze the performance of \CBE~policy, we establish the following three facts:

\emph{Fact 1.} The probability that the best arm is eliminated at a step $n$ by a suboptimal arm is upper bounded by $\frac{K}{T^4}$: for $k\neq*$,
\begin{eqnarray}\nn
&&\hspace{-3em}\Pr\left[\MVbar_{*}[u_n]- \frac{\Gahat_n}{4}> \min_{j\in\K_n}\MVbar_{j}[u_n] +\frac{\Gahat_n}{4}\right] \\\nn
&\le& \sum_{k\in\K_n\setminus k} \Pr\left[\MVbar_{*}[u_n]- \frac{\Gahat_n}{4}> \MVbar_{k}[u_n] +\frac{\Gahat_n}{4}\right] \\\nn
&\le& \sum_{k\in\K_n\setminus k} \Pr\left[\MVbar_{*}[u_n]- \MV_* > \frac{\Gahat_n}{4}~\text{or}~\MVbar_{k}[u_n] -\MV_k<-\frac{\Gahat_n}{4}\right] \\\nn
&\le& \Pr\left[\MVbar_{*}[u_n]- \MV_* > \frac{\Gahat_n}{4}\right] + \sum_{k\in\K_n\setminus k} \Pr\left[\MVbar_{k}[u_n] -\MV_k<-\frac{\Gahat_n}{4}\right] \\\nn
&\le& K\exp\left(-\frac{\alpha u_n\Gahat^2}{16} \right)\\ \nn
&\le&\frac{K}{T^4}.
\end{eqnarray}

\emph{Fact 2.} Conditioned on the probability that the optimal arm is not eliminated, the probability that a suboptimal arm $k$ is not eliminated at step $n$ where $\Gahat_n<\Gamma_k$
is also upper bounded by $\frac{2}{T^4}$: for $k\neq*$,
\begin{eqnarray}\nn
&&\hspace{-3em}\Pr\left[\MVbar_{k}[u_n]- \frac{\Gahat_n}{4}< \min_{j\in\K_n}\MVbar_{j}[u_n] +\frac{\Gahat_n}{4}\right] \\\nn
&\le& \Pr\left[\MVbar_{k}[u_n]- \frac{\Gahat_n}{4}< \MVbar_{*}[u_n] +\frac{\Gahat_n}{4}\right] \\\nn
&\le& \Pr\left[\MVbar_{*}[u_n]- \MV_* > \frac{\Gahat_n}{4}~\text{or}~\MVbar_{k}[u_n] -\MV_k<-\frac{\Gahat_n}{4}\right] \\\nn
&\le& \Pr\left[\MVbar_{*}[u_n]- \MV_* > \frac{\Gahat_n}{4}\right] +  \Pr\left[\MVbar_{k}[u_n] -\MV_k<-\frac{\Gahat_n}{4}\right] \\\nn
&\le& 2\exp\left(-\frac{\alpha u_n\Gahat^2}{16} \right) \\ \nn
&\le&\frac{2}{T^4}.
\end{eqnarray}

\emph{Fact 3.} Conditioned on the probability that the optimal arm is not eliminated, the probability that a suboptimal arm is eliminated at a step $n$ where $\Gahat_n>4\Gamma_k$
is upper bounded by $\frac{K}{T}$: for $k\neq*$,
\begin{eqnarray}\nn
&&\hspace{-3em}\Pr\left[\MVbar_{k}[u_n]- \frac{\Gahat_n}{4}> \min_{j\in\K_n}\MVbar_{j}[u_n] +\frac{\Gahat_n}{4}\right] \\\nn
&\le& \sum_{j\in\K_n\setminus k} \Pr\left[\MVbar_{k}[u_n]- \frac{\Gahat_n}{4}> \MVbar_{j}[u_n] +\frac{\Gahat_n}{4}\right] \\\nn
&\le& \sum_{j\in\K_n\setminus k} \Pr\left[\MVbar_{k}[u_n]- \MV_k > \frac{\Gahat_n}{4}~\text{or}~\MVbar_{j}[u_n] -\MV_j<-\frac{\Gahat_n}{4}\right] \\\nn
&\le& \Pr\left[\MVbar_{k}[u_n]- \MV_k > \frac{\Gahat_n}{8}\right] + \sum_{j\in\K_n\setminus k} \Pr\left[\MVbar_{j}[u_n] -\MV_j<-\frac{\Gahat_n}{8}\right] \\\nn
&\le& K\exp\left(-\frac{\alpha u_n\Gahat^2}{64} \right) \\ \nn
&\le&\frac{K}{T}.
\end{eqnarray}

Let $n_k=\min\{n:\Gahat_n\le\Gamma_k\}$ and let $n_{\max}$ be the total number of steps at time $T$. Clearly, $n_{\max}\le \log_2 T$. Using Facts 1 and 2, we have, for $k\neq*$,
\begin{eqnarray}\nn
\E[\tau_{k,T}]&\le& 
\sum_{m=1}^{n_k} u_n + (\Pr[\text{the best arm is eliminated by a suboptimal arm}])T\\\nn 
&+& (\Pr[\text{arm $k$ is not eliminated at (or before) step $n_k$}])T\\\nn
&\le&\sum_{m=1}^{n_k}(\frac{C\log T}{\Gahat_m^2}+1) + \frac{Kn_{\max}+2}{T^3} \\\nn
&\le& \sum_{m=1}^{n_k}\frac{C\log T}{\Gamma_k^2}(\frac{1}{4}^{m-1} + 1) + \frac{K\log_2T+2}{T^3}\\\nn
&\le&
\frac{\frac{4C}{3}\log T}{\Gamma_k^2} + \log_2(\frac{1}{\Gamma_k})+ \frac{K\log_2T+2}{T^3}.
\end{eqnarray}

The first term in regret expression given in~\eqref{RegUB} is thus upper bounded by

\begin{eqnarray}\label{sh1}
\Sumnotkstar (\frac{\frac{4C}{3}\log T}{\Gamma_k^2}+\log_2(\frac{1}{\Gamma_k}) +\frac{K\log_2T+2}{T^3}\wedge T) \Gamma_k
\end{eqnarray}

The second term of the upper bound on regret given in~\eqref{SO} is upper bounded by
\begin{eqnarray}\nn
&&\hspace{-2em}(K-1)\Sumnotkstar\sum_{t=1}^T  \{\Pr[\pi_{t}=k]\wedge\frac{1}{4}\}\Delta_k^2\\\nn &\le& 
(K-1)\Sumnotkstar u_{n_k-1}\I[n_k-1\le n_{\max}](n_{\max}-n_k+1)\frac{2\Delta_{\max}^2}{4}
\\\nn
&+&
(K-1)\Sumnotkstar u_{n_k21}\I[n_k-2\le n_{\max}](n_{\max}-n_k+2)\frac{2\Delta_{\max}^2}{4}\\\nn
&+&
\bigg(\Pr[\text{the best arm is eliminated by a suboptimal arm}]\\\nn 
&+& 
\Pr[\text{arm $k$ is not eliminated at (or before) step $n_k$}] \\ \nn 
&+&
\Pr[\text{arm $k$ is eliminated at (or before) step $n_k-3$}]
\bigg)\frac{T\Delta_{\max}^2}{4}\\\nn
&=&
\frac{1}{2}\log_2T\Delta_{max}^2\Sumnotkstar\bigg( \bigg(\frac{C\log T}{\Gamma_k^2}+1\bigg)\I[n_k\le n_{\max}]\\\nn
&+&
\bigg(\frac{\frac{C}{4}\log T}{\Gamma_k^2}+1\bigg)\I[n_k-1\le n_{\max}]\bigg)\\\label{sh2}
&&~~~+(\frac{K\log_2 T+2}{T^4}+\frac{K\log_2 T}{T})(\frac{(K-1)^2T\Delta_{\max}^2}{4}).
\end{eqnarray}
Combining~\eqref{sh1} and~\eqref{sh2}, we arrive at the theorem. 

\end{proof}

\end{document}